\documentclass{article}

\usepackage[final]{nips_2018}

\usepackage[utf8]{inputenc} 
\usepackage[T1]{fontenc}    
\usepackage{url}            
\usepackage{booktabs}       
\usepackage{amsfonts}       
\usepackage{nicefrac}       
\usepackage{microtype}      

\usepackage{listings}
\usepackage{tikz}
\usetikzlibrary{arrows}
\usepackage{subcaption}
\usepackage[english]{babel}
\usepackage{graphicx,caption,centernot,amssymb,amsthm,algpseudocode,mathtools}
\usepackage{amsmath}
\usepackage{bm}
\usepackage{bbm}
\usepackage{paralist}
\usepackage{comment}

\newcommand{\defeq}{\vcentcolon\protect\nolinebreak\mkern-1.2mu=}

\DeclarePairedDelimiterX\set[1]\lbrace\rbrace{\def\given{\;\delimsize\vert\;}#1}
\newcommand\mathbfsc[1]{\text{\normalfont\scshape#1}}
\newcommand\pa[1]{\mathbfsc{pa}(#1)}

\newcommand\de[1]{\mathbfsc{de}(#1)}

\definecolor{blue}{rgb}{0,0,1}
\definecolor{orange}{rgb}{1,0.5,0.1}
\definecolor{brightblue}{rgb}{0.92,0.92,1}
\definecolor{brightgrey}{rgb}{0.96,0.96,1}
\definecolor{green}{rgb}{0.2,1.0,0.2}

\tikzstyle{var}=[circle,draw=black,fill=white,thin,minimum size=18pt,inner sep=0pt]
\tikzstyle{varh}=[circle,draw=gray,fill=white,thin,minimum size=18pt,inner sep=0pt,dashed]
\tikzstyle{vartarget}=[circle,draw=black,fill=lightgray,thin,minimum size=18pt,inner sep=0pt]
\tikzstyle{varintervention}=[rectangle,draw=black,fill=white,thin,minimum size=18pt,inner sep=0pt]
\tikzstyle{arr}=[->,>=stealth',draw=black,thick]
\tikzstyle{arrh}=[->,>=stealth',draw=gray,thick,dashed]
\tikzstyle{biarr}=[<->,>=stealth',draw=black,fill=black,thick]
\tikzstyle{biarrh}=[<->,>=stealth',draw=gray,fill=gray,thick]

\lstdefinestyle{ASP}{
  belowcaptionskip=1\baselineskip,
  breaklines=true,
  frame=L,
  xleftmargin=\parindent,
  language=psl,
  showstringspaces=false,
  basicstyle=\script\ttfamily,
  keywordstyle=\bfseries\color{green!40!black},
  commentstyle=\itshape\color{purple!40!black},
  identifierstyle=\color{blue},
  stringstyle=\color{orange},
}

\newcommand\indep{{\,\perp\mkern-12mu\perp\,}}
\newcommand\notindep{{\,\not\mkern-1mu\perp\mkern-12mu\perp\,}}
\newcommand\dep{{\,\not\mkern-1mu\perp\mkern-12mu\perp\,}}
\newcommand{\dsep}{\perp}

\newcommand\idcausal{\dashrightarrow}
\newcommand\idacausal{\text{${}\not\!\dashrightarrow{}$}}
\newcommand\B[1]{\bm{#1}}
\newcommand\C[1]{\mathcal{#1}}
\newcommand\BC[1]{\bm{\mathcal{#1}}}
\newcommand\given{\,|\,}

\newcommand\causes{\idcausal}
\newcommand\notcauses{\idacausal}

\newcommand\eref[1]{(\ref{#1})}

\newtheorem{proposition}{Proposition}

\newtheorem{assumption}{Assumption}
\newtheorem{task}{Task}

\newtheorem{example}{Example}

\makeatletter
\newcommand{\setassumptiontag}[1]{
  \let\oldtheassumption\theassumption
  \renewcommand{\theassumption}{#1}
  \g@addto@macro\endassumption{
  \addtocounter{assumption}{-1}
  \global\let\theassumption\oldtheassumption}
  }
\makeatother

\makeatletter
\renewenvironment{proof}[1][\proofname]{\par
  \pushQED{\qed}%
  \normalfont \topsep6\p@\@plus6\p@\relax
  \trivlist
  \item[\hskip\labelsep
        \bfseries
    #1\@addpunct{.}]\ignorespaces
}{%
  \popQED\endtrivlist\@endpefalse
}
\makeatother
\newcommand\Prb{\mathbb{P}}
\newcommand\Exp{\mathbb{E}}

\definecolor{lgray}{rgb}{0.9,0.9,0.9}

\graphicspath{{images/}}

\title{Domain Adaptation by Using Causal Inference to Predict Invariant Conditional Distributions}

\author{
  Sara Magliacane\\
  IBM Research\thanks{Most of the work was performed while at the University of Amsterdam.}\\
  \texttt{sara.magliacane@gmail.com}
  \And
  Thijs van Ommen\\
  University of Amsterdam\\
  \texttt{thijsvanommen@gmail.com}
  \And
  Tom Claassen\\
  Radboud University Nijmegen\\
  \texttt{tomc@cs.ru.nl} 
  \And
  Stephan Bongers\\
  University of Amsterdam\\
  \texttt{srbongers@gmail.com}
  \And
  Philip Versteeg\\
  University of Amsterdam\\
  \texttt{p.j.j.p.versteeg@uva.nl}
  \And
  Joris M.\ Mooij\\
  University of Amsterdam\\
  \texttt{j.m.mooij@uva.nl}   
}

\begin{document}

\maketitle

\begin{abstract}
An important goal common to domain adaptation and causal inference is to make accurate predictions when the distributions for the source (or training) domain(s) and target (or test) domain(s) differ. In many cases, these different distributions can be modeled as different contexts of a single underlying system, in which each distribution corresponds to a different perturbation of the system, or in causal terms, an intervention.
We focus on a class of such \emph{causal} domain adaptation problems, 
where data for one or more source domains are given, and the task is to predict the distribution of a certain target variable from measurements of other variables in one or more target domains.
We propose an approach for solving these problems that exploits causal inference and does not rely on prior knowledge of the causal graph, the type of interventions or the intervention targets.
We demonstrate our approach by evaluating a possible implementation on simulated and real world data.
\end{abstract}

\section{Introduction}\label{sec:intro}

Predicting unknown values based on observed data is a problem central to many sciences, and well studied in statistics and machine learning. This problem becomes significantly harder if the training and test data do not have the same distribution, for example because they come from different domains. Such a distribution shift can happen whenever the circumstances under which the training data were gathered are different from those for which the predictions are to be made. A rich literature exists on this problem of \emph{domain adaptation}, a particular task in the field of \emph{transfer learning}; see e.g.~\citet{datasetshiftML2009,Pan2010survey} for overviews.

When the domain changes, so may the relations between the different variables under consideration. While for some sets of variables $\B{A}$, a function $f:\BC{A}\to\C{Y}$ learned in one domain may continue to offer good predictions for $Y\in\C{Y}$ in a different domain, this may not be true of other sets $\B{A}'$ of variables.
\emph{Causal graphs}  \citep[e.g.,][]{Pearl2009,Spirtes2000} allow us to reason about this in a principled way when the domains correspond to different external \emph{interventions} on the system, or more generally, to different contexts in which a system has been measured. Knowledge of the causal graph that describes the data generating mechanism, and of which parts of the model are invariant across the different domains, allows one to transfer knowledge from one domain to the other in order to address the problem of domain adaptation \citep{Spirtes2000,Storkey2009,Schoelkopf_et_al_ICML_12,BareinboimPearl_2016}.

Over the last years, various methods have been proposed to exploit the causal structure of the data generating process in order to address certain domain adaptation problems,
each relying on different assumptions. For example, \citet{BareinboimPearl_2016} provide
theory for identifiability under transfer (``transportability'') assuming that the causal graph is
known, that interventions are perfect, and that the intervention targets are known.
\citet{Hyttinen++_2015} also assume perfect interventions with known targets but do not rely on complete knowledge of the causal graph, instead inferring the relevant aspects of it from the data. \citet{Rojas-Carulla++_2018} make the assumption that if the conditional distribution of the target given some subset of covariates is invariant across different source domains, then this conditional distribution must also be the same in the target domain. The methods proposed in \citep{Schoelkopf_et_al_ICML_12,ZhangSMW2013,ZhangGS2015,Gong++2016} all address challenging settings in which conditional independences that follow from the usual Markov and faithfulness assumptions alone do not suffice to solve the problem, but additional assumptions on the data generating process have to be made. 

In this work, we will make no such additional assumptions, and address the setting in which both the causal graph \emph{and} the intervention types and targets may be (partially) unknown. 
Our contributions are the following.
We consider a set of relatively weak assumptions that make the problem well-posed. 
We propose an approach to solve this class of causal domain adaptation problems that can deal
with the presence of latent confounders. The main idea is to select
the subset of features $\B{A}$ that leads to the best predictions of $Y$ in the source domains,
while satisfying \emph{invariance} (i.e., $\Prb(Y \given \B{A})$ is the same in the source and target domains).
To test whether the invariance condition is satisfied, we apply the recently proposed Joint Causal Inference
(JCI) framework \citep{Mooij++_1611.10351v3} to
exploit the information provided by multiple domains
corresponding to different interventions.
The basic idea is as follows.
First, a standard feature selection method is
applied to source domains data to find sets of features that are predictive of a target variable, trading off bias and 
variance, but unaware of changes in the distribution across domains. A causal inference method then
draws conclusions from all given data about the possible causal graphs, avoiding sets of features for which
the predictions would not transfer to the target domains.
We propose a proof-of-concept implementation of our approach building on a causal discovery algorithm by \citet{antti}.
We evaluate the method on synthetic data and a real-world example.





\section{Theory}

\begin{figure}[t]\centering
  \begin{subfigure}[b]{0.17\textwidth}%
    \centering\captionsetup{width=.85\linewidth}%
    \begin{tikzpicture}
      \node[var] (I1) at (-1,2) {$C_1$};
      \node[var] (A) at (0,1) {$X_1$};
      \node[var] (Y) at (0,0) {$X_2$};
      \node[var] (Z) at (0,-1) {$X_3$};
      \draw[arr] (I1) edge (A);
      \draw[arr,bend right] (I1) edge (Z);
      \draw[arr] (A) edge (Y);
      \draw[arr] (Y) edge (Z);
      \draw[dashed] (-1.5,1.5) -- (0.5,1.5);
    \end{tikzpicture}
    \caption{Causal graph}
  \end{subfigure}\quad%
  \begin{subfigure}[b]{0.385\textwidth}%
    \centering\captionsetup{width=\linewidth}%
    \includegraphics[width=0.75\textwidth]{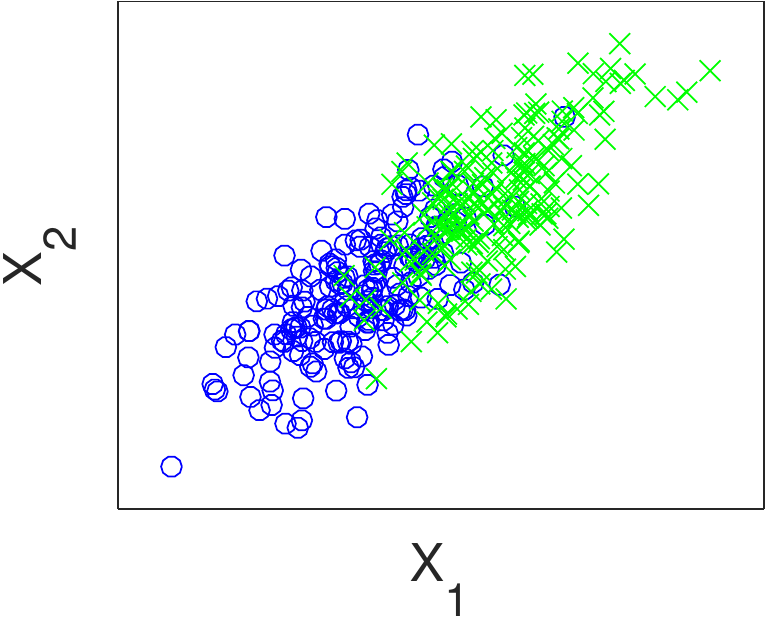}%
    \caption{No distribution shift for $\set{X_1}$: $\Prb(Y \given X_1, C_1=0)=\Prb(Y \given X_1, C_1=1)$}
  \end{subfigure}\quad%
  \begin{subfigure}[b]{0.385\textwidth}%
    \centering\captionsetup{width=\linewidth}%
    \includegraphics[width=0.75\textwidth]{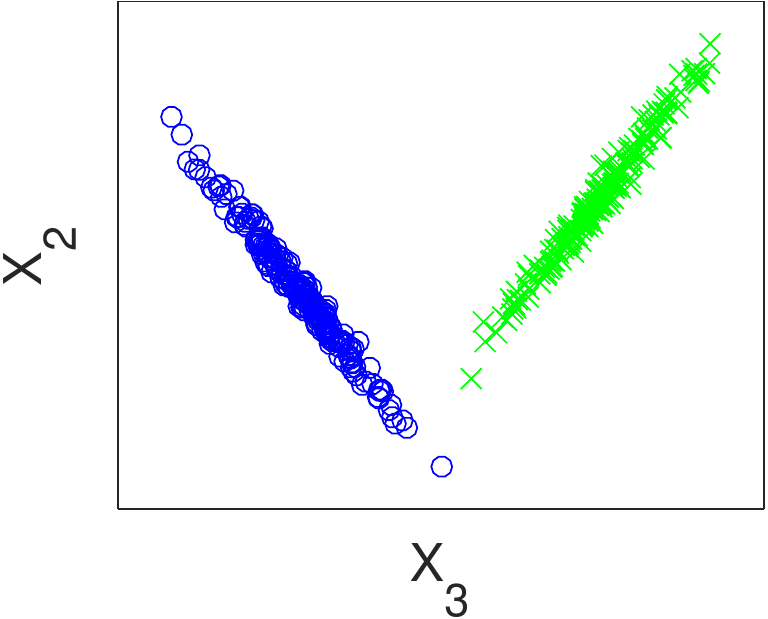}%
    \caption{Strong distribution shift for $\set{X_3}$: $\Prb(Y \given X_3, C_1=0)\ne\Prb(Y \given X_3, C_1=1)$}
  \end{subfigure}%
  \caption{In this scenario, an intervention $C_1$ leads to a shift of distribution between source domain and target domain (see also Example~\ref{ex:FSfailure}). Green crosses show source domain data ($C_1=0$), blue circles show target domain data ($C_1=1$). A standard feature selection method that does not take into account the causal structure, but would use $X_3$ to predict $Y \defeq X_2$ (because $X_3$ is a good predictor of $Y$ in the source domain), would obtain extremely biased predictions in the target domain. Using $X_1$ instead yields less accurate predictions in the source domain, but much more accurate ones in the target domain.
\label{fig:FSfailure}}
\end{figure}
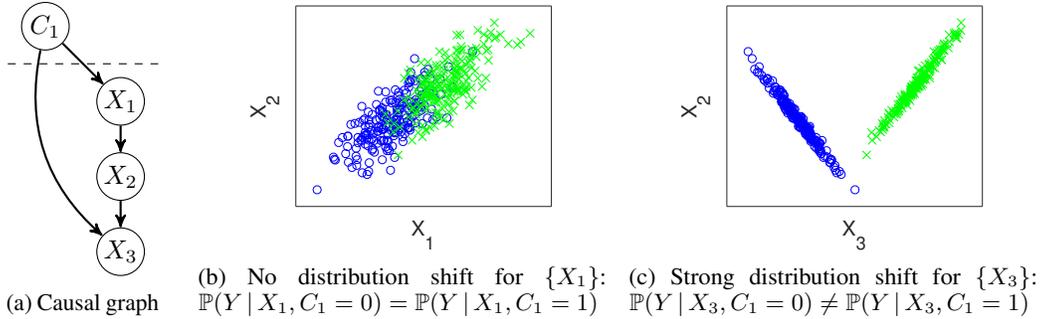
Before giving a precise definition of the class of domain adaptation problems that
we consider in this work, we begin with a motivating example. 
\begin{example}\label{ex:FSfailure}
We are given three variables $X_1,X_2,X_3$ describing different aspects of a system (for example, certain blood cell phenotypes in mice).
We have observational measurements of these three variables (the source domain, designated with $C_1=0$), and in addition, measurements of $X_1$ and $X_3$ under an intervention (the target domain, designated with $C_1=1$), e.g., in which the mice have been exposed to a certain drug. The domain adaptation task is to predict the values of $Y \defeq X_2$ in the interventional target domain (i.e., when $C_1=1$).
Let us assume for this example that the causal graph in Figure~\ref{fig:FSfailure}a applies, i.e., we assume that $X_2$ is affected by $X_1$ and affects $X_3$, while $C_1$ affects both $X_1$ and $X_3$ (i.e., the intervention targets the variables $X_1$ and $X_3$).
This causal graph implies $\Prb(Y \given X_1,C_1=0) = \Prb(Y \given X_1,C_1=1)$.
Suppose further that the relation between $X_1$ and $X_2$ is about equally strong as the relation between $X_2$ and $X_3$, but considerably more noisy.
  Then a feature selection method using only available source domain data, and aiming to select the best subset of features to use for prediction of $Y$ 
  will prefer both $\{X_3\}$ and $\set{X_1,X_3}$ over $\{X_1\}$ (because predicting $Y$ from $X_1$ leads to larger variance than predicting $Y$ from $X_3$, and to a larger bias than predicting $Y$ from both $X_1$ and $X_3$).
  However, under the intervention ($C_1 = 1$), $\Prb(Y \given X_3)$ and $\Prb(Y \given X_1,X_3)$ both change,\footnote{More precisely, we should say that $\Prb(Y \given X_3, C_1=1)$ may differ from $\Prb(Y \given X_3, C_1=0)$, and similarly when conditioning on $\{X_1,X_3\}$.} so that using those features to predict $Y$ in the target domain could lead to extreme bias, as illustrated in Figure~\ref{fig:FSfailure}c. Because the conditional distribution of $Y$ given $X_1$ is invariant across domains, as illustrated in Figure~\ref{fig:FSfailure}b, predictions of $Y$ based only on $X_1$ can be safely transferred to the target domain.
\end{example}
This example provides an instance of a domain adaptation problem where feature selection methods that do not take into account the causal structure would pick a set of features that does not generalize to the target domain, and may lead to arbitrarily bad predictions (even asymptotically, as the number of data points tends to infinity). On the other hand, correctly taking into account the causal structure and the possible distribution shift from source to target domain allows to upper bound the prediction error in the target domain, as we will see in Section~\ref{sec:causal_feature_selection}.

\subsection{Problem Setting}
We now formalize the domain adaptation problems that we address in this paper. 
We will make use of the terminology of the recently proposed 
Joint Causal Inference (JCI) framework \citep{Mooij++_1611.10351v3}. 

Let us consider a system of interest described by a set of \emph{system
variables} $\{X_j\}_{j\in\C{J}}$. In addition, we model the domain in which the
system has been measured by \emph{context variables} $\{C_i\}_{i\in\C{I}}$ (we will
use ``context'' as a synonym for ``domain'').
We will denote the tuple of all system and context variables as 
$\B{V} = ((X_j)_{j\in\C{J}},(C_i)_{i\in\C{I}})$.
System and context variables can be discrete or continuous. As a concrete
example, the system of interest could be a mouse. The system variables could be
blood cell phenotypes such as the concentration of red blood cells, the
concentration of white blood cells, and the mean red blood cell volume. The
context variables could indicate for example whether a certain gene has been
knocked out, the dosage of a certain drug administered to the mice, the age
and gender of the mice, or the lab in which the measurements were done. The
important underlying assumption is that context variables are \emph{exogenous}
to the system, whereas system variables are \emph{endogenous}. The
interventions are not limited to the perfect (``surgical'') interventions
modeled by the do-operator of \citet{Pearl2009}, but can also be other types of
interventions such as mechanism changes \citep{TianPearl2001}, soft interventions 
\citep{Markowetz++2005}, fat-hand interventions \citep{EatonMurphy07}, 
activity interventions \citep{MooijHeskes_UAI_13}, and stochastic versions of all these. 
Knowledge of the intervention \emph{targets} is not necessary (but is certainly helpful). 
For example, administering
a drug to the mice may have a direct causal effect on an unknown subset of the system variables,
but we can simply model it as a binary exogenous variable (indicating whether or not the drug
was administered) or a continuous exogenous variable (describing
the dosage of the administered drug) without specifying in advance on which variables it has a direct effect.
We can now formally state the domain adaptation task that we address in this work:
\begin{task}[Domain Adaptation Task]\label{task:causaldomainadaptation}
We are given data for a single or for multiple source domains, in each of which $C_1 = 0$, and for a single or for multiple target domains,
in each of which $C_1 = 1$. Assume the source domains data is complete (i.e., no missing values), and the target
domains data is complete with the exception of all values of a certain target variable $Y = X_j$.
The task is to predict these missing values of the target variable $Y$ given the available source and target domains data. 
\end{task}
An example is provided in Figure~\ref{fig:task}. In the next subsection, we will formalize our assumptions
to turn this task into a well-posed problem.

\subsection{Assumptions}

Our first main assumption is that the data generating process (on both system
and context variables) can be represented as a Structural Causal Model (SCM)
(see e.g., \citep{Pearl2009}):
\begin{equation}\label{eq:SCM_JCI}
  \C{M}:
  \begin{cases}
    C_i &= g_i(\B{E}_{\pa{i} \cap \C{K}}), \qquad i \in \C{I} \\
    X_j &= f_j( \B{X}_{\pa{j} \cap \C{J}}, \B{C}_{\pa{j} \cap \C{I}}, \B{E}_{\pa{j} \cap \C{K}}), \qquad j \in \C{J}\\
    p(\B{E}) & = \prod_{k\in\C{K}} p(E_k).
  \end{cases}
\end{equation}
Here, we introduced exogenous latent independent ``noise'' variables $(E_k)_{k\in\C{K}}$ that
model latent causes of the context and system variables. The parents of each variable are 
denoted by $\pa{\cdot}$. Each context and system variable is related to its parent variables 
by a structural equation. In addition, we assume a factorizing probability distribution on the
exogenous variables.
There could be cyclic dependencies, for example due to feedback loops, but for simplicity of exposition
we will discuss only the acyclic case here, noting that the extension to the cyclic
case is straightforward given recent theoretical advances on cyclic SCMs \citep{Bongers++_1611.06221v2}.
This SCM provides a causal model for the distributions of the various domains, and in particular,
it induces a joint distribution $\Prb(\B{V})$ on the context and system variables.
Note that we will assume that the data generating process can be modeled by some model of this form,
but we do not rely on knowing the precise model. 

The SCM $\C{M}$ can be represented graphically by its causal graph $\C{G}(\C{M})$, a graph with
nodes $\C{I} \cup \C{J}$ (i.e., the labels of both system and context variables), directed edges 
$l_1 \to l_2$ for $l_1,l_2 \in \C{I} \cup \C{J}$ iff $l_1 \in \pa{l_2}$, and bidirected edges 
$l_1 \leftrightarrow l_2$ for $l_1,l_2 \in \C{I} \cup \C{J}$ iff there exists a $k \in \pa{l_1} \cap \pa{l_2} \cap \C{K}$. In the acyclic
case, this causal graph is an Acyclic Directed Mixed Graph (ADMG), and $\C{M}$ is also known as a 
Semi-Markov Causal Model (see e.g., \citep{Pearl2009}). The directed edges represent direct causal
relationships, and the bidirected edges may represent hidden confounders (both relative to the set of
variables in the ADMG).
The (causal) \emph{Markov assumption} holds \citep{Richardson2003}, i.e., any d-separation $\B{A} \dsep \B{B} \given \B{S}\ [\C{G}(\C{M})]$ between sets of random variables $\B{A},\B{B},\B{S} \subseteq \B{V}$ in the ADMG $\C{G}(\C{M})$ implies a conditional independence $\B{A} \indep \B{B} \given \B{S}\ [\Prb(\B{V})]$ in the distribution $\Prb(\B{V})$ induced by the SCM $\C{M}$.
A standard assumption in causal discovery is that the joint distribution $\Prb(\B{V})$ is \emph{faithful} with respect to the ADMG $\C{G}(\C{M})$, i.e., that there are no other conditional independences in the joint distribution than those implied by d-separation.

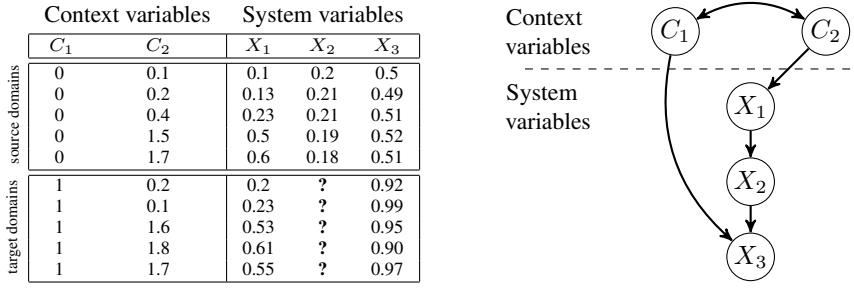
\begin{figure}

\begin{tikzpicture}
\begin{scope}[xshift=0cm]
  \node at (0,-0.5) {\scriptsize\begin{tabular}[t]{|cc|ccc|}
    \multicolumn{2}{c}{\small Context variables} & \multicolumn{3}{c}{\small System variables} \\[2pt]
    \hline
    \rule{0pt}{2ex} $C_1$ & $C_2$ & $X_1$ & $X_2$ & $X_3$ \\
    \hline
    \hline
    0 & 0.1 & 0.1  & 0.2 & 0.5\\
    0 & 0.2 & 0.13 & 0.21 & 0.49\\
    0 & 0.4 & 0.23 & 0.21 & 0.51\\
    0 & 1.5 & 0.5 & 0.19 & 0.52\\
    0 & 1.7 & 0.6 & 0.18 & 0.51\\
    \hline
    \hline
    1 & 0.2 & 0.2 &  \textbf{?} & 0.92\\
    1 & 0.1 & 0.23 & \textbf{?} & 0.99\\
    1 & 1.6 & 0.53 & \textbf{?} & 0.95\\
    1 & 1.8 & 0.61 & \textbf{?} & 0.90\\
    1 & 1.7 & 0.55 & \textbf{?} & 0.97\\
    \hline
  \end{tabular}};
  \node[rotate=90] at (-2.8,-0.1) {\tiny source domains};
  \node[rotate=90] at (-2.8,-1.6) {\tiny target domains};
\end{scope}
\begin{scope}[xshift=5cm,yshift=0.5cm]
   
\end{scope}
    \begin{scope}[xshift=7cm,yshift=-1cm]
      \node[text width=1.5cm] at (-2.5,2) {\small Context variables};
      \node[text width=1.5cm] at (-2.5,1) {\small System variables};
      \node[var] (I1) at (-1,2) {$C_1$};
      \draw[dashed] (-3,1.5) -- (1.5,1.5);
      \node[var] (I2) at (1,2) {$C_2$};
      \node[var] (A) at (0,1) {$X_1$};
      \node[var] (Y) at (0,0) {$X_2$};
      \node[var] (Z) at (0,-1) {$X_3$};
      \draw[arr,bend right] (I1) edge (Z);
      \draw[arr] (I2) edge (A);
      \draw[arr] (A) edge (Y);
      \draw[arr] (Y) edge (Z);
      \draw[biarr,bend left] (I1) edge (I2);
    \end{scope}
\end{tikzpicture}
\caption{Example of a causal domain adaptation problem. The causal graph is depicted on the right, the corresponding data on the left. The task is to predict the missing values of $Y = X_2$ in the target domains $(C_1=1)$, based on the observed data from the source domains and the target domains, without knowledge of the causal graph. See also Example~\ref{ex:identifiable_but_fs_fails}.\label{fig:task}}
\end{figure}

We will make the following assumptions on the causal structure
(where henceforth we will simply write $\C{G}$ instead of $\C{G}(\C{M})$),
which are discussed in detail by \citet{Mooij++_1611.10351v3}:
\begin{assumption}[JCI Assumptions]\label{ass:jci_backgroundknowledge}
  Let $\C{G}$ be a causal graph with variables $\B{V}$ (consisting of system variables $\{X_j\}_{j\in\C{J}}$ and context variables $\{C_i\}_{i\in\C{I}}$).
\begin{compactenum}[(i)]
\item No system variable directly causes any context variable (``exogeneity'') \\ ($\forall j \in \C{J}, \forall i \in \C{I}: X_j \rightarrow C_i \notin \C{G}$);
\item No system variable is confounded with a context variable (``randomization'') \\ ($\forall j\in\C{J}, \forall i\in\C{I}: X_j \leftrightarrow C_i \notin \C{G}$);
\item Every pair of context variables is purely confounded (``genericity'') \\ ($\forall i, i' \in \C{I}: C_i \leftrightarrow C_{i'} \in \C{G} \land C_{i} \rightarrow C_{i'} \notin \C{G}$).
\end{compactenum}
\end{assumption}
The first assumption is the most crucial one that captures what we mean
by ``context''. The other two assumptions are less crucial and could be
omitted, depending on the application. For a more in-depth discussion of
these modeling assumptions and on how they compare with other possible causal
modeling approaches, we refer the reader to \citep{Mooij++_1611.10351v3}.
Any causal discovery method can in principle be used in the JCI
setting, but identifiability greatly benefits from taking into account the
background knowledge on the causal graph from
Assumption~\ref{ass:jci_backgroundknowledge}.

In addition, in order to be able to address the causal domain adaptation task, we will assume:
\begin{assumption}\label{ass:ctl}
  Let $\C{G}$ be a causal graph with variables $\B{V}$ (consisting of system variables $\{X_j\}_{j\in\C{J}}$ and context variables $\{C_i\}_{i\in\C{I}}$), and $\Prb(\B{V})$ be the corresponding distribution on $\B{V}$. Let $C_1$ be the source/target domains indicator and $Y = X_j$ the target variable.
  \begin{compactenum}[(i)]
  \item The distribution $\Prb(\B{V})$ is Markov and faithful w.r.t.\ $\C{G}$;\label{ass:ctl_markov_faithful}
    \item Any conditional independence involving $Y$ in the source domains also holds in the target domains, i.e., if $\B{A} \cup \B{B} \cup \B{S}$ contains $Y$ but not $C_1$ then:\footnote{Here, with $\B{A} \indep \B{B} \given \B{S}\ [C_1=0]$ we mean $\B{A} \indep \B{B} \given \B{S}\ [\Prb(\B{V} \given C_1=0)]$, i.e., the conditional independence of $\B{A}$ from $\B{B}$ given $\B{S}$ in the mixture of the source domains $\Prb(\B{V} \given C_1=0)$, and similarly for the target domains.}\label{ass:ctl_faithfulness_source}
      $$\B{A} \indep \B{B} \given \B{S}\ [C_1 = 0] \implies \B{A} \indep \B{B} \given \B{S}\ [C_1 = 1];$$
    \item $C_1$ has no \emph{direct} effect on $Y$ w.r.t.\ $\B{V}$, i.e., $C_1\to Y \notin \C{G}$.
    \label{ass:ctl_no_direct_effect}
  \end{compactenum}
\end{assumption}
The Markov and faithfulness assumptions are standard in constraint-based causal discovery on a single domain; we apply them here
on the ``meta-system'' composed of system and context.

Assumption~\ref{ass:ctl}(\ref{ass:ctl_faithfulness_source}) may seem non-intuitive, but as we show in the Supplementary Material, it follows from  more intuitive (but stronger) assumptions, for example if both the pooled source domains distribution $\Prb(\B{V} \given C_1 = 0)$ and the pooled target domains distribution $\Prb(\B{V} \given C_1 = 1)$ are Markov and faithful to the subgraph of $\C{G}$ which excludes $C_1$. These stronger assumptions imply that the causal structure (i.e., presence or absence of
causal relationships and confounders) of the other variables is invariant when going from source to target domains.
%
Assumption~\ref{ass:ctl}(\ref{ass:ctl_faithfulness_source}) is a weakened version of these more natural assumptions, allowing additional \emph{independences} to hold in the target domains compared to the source domains, e.g., when $C_1$ models a perfect surgical intervention.

Assumption~\ref{ass:ctl}(\ref{ass:ctl_no_direct_effect}) is strong, yet some assumption of that type
seems necessary to make the task well-defined. Without any information
at all about
the target(s) of $C_1$, or
the causal mechanism that determines the values of $Y$ in the target domains, predicting the values of $Y$ for the target domains seems generally impossible.
Note that the assumption is more likely to be satisfied if the interventions are believed to be precisely targeted, and gets weaker the more relevant system variables are observed.\footnote{This assumption can be weakened further: in some circumstances one can infer from the data and the other assumptions that $C_1$ cannot have a direct effect on $Y$.
For example: if there exists a descendant $D \in \de{Y}$, and if there exists a set $\B{S} \subseteq \B{V} \setminus (\{C_1, Y\} \cup \de{Y})$, such that $C_1 \indep D \given \B{S}$, then $C_1$ is not a direct cause of $Y$ w.r.t.\ $\B{V}$. For some proposals on alternative assumptions that can be made when this assumption is violated, see e.g., \citep{Schoelkopf_et_al_ICML_12,ZhangSMW2013,ZhangGS2015,Gong++2016}.}

As one example of a real-world setting in which these assumptions are reasonable, consider a genomics experiment, in which gene
expression levels of many different genes are measured in response to knockouts
of single genes. Given our present-day understanding of the biology of gene
expression, it is very reasonable to assume that the knockout of gene $X_i$ only
has a direct effect on the expression level of gene $X_i$ itself. As long as we
do not ask to predict the expression level of $X_i$ under a knockout of $X_i$, 
but only the expression level of other genes $Y=X_j$ with $j \ne i$,
Assumption~\ref{ass:ctl}(\ref{ass:ctl_no_direct_effect}) seems justified. 
%
%
It is also reasonable (based on present-day understanding of biology) to expect that a single gene knockout does not change the causal mechanisms in the rest of the system. This justifies Assumption~\ref{ass:ctl}(\ref{ass:ctl_faithfulness_source}) in this setting if one is willing to assume faithfulness. 

In the next subsections, we will discuss how these assumptions enable us to address the domain adaptation task.

\subsection{Separating Sets of Features}\label{sec:causal_feature_selection}

Our approach to addressing Task~\ref{task:causaldomainadaptation} is based on finding a \emph{separating set} $\B{A} \subseteq \B{V} \setminus \{C_1,Y\}$ of (context and system) variables that satisfies $C_1 \dsep Y \given \B{A}\ [\C{G}]$.
If such a separating set $\B{A}$ can be found, then the distribution of $Y$ conditional on $\B{A}$ is \emph{invariant} under transferring from the source domains to the target domains, i.e., $\Prb(Y \given \B{A}, C_1 = 0) = \Prb(Y \given \B{A}, C_1 = 1)$.
As the former conditional distribution can be estimated from the source domains data, we directly obtain a prediction for the latter, which then enables us to predict the values of $Y$ from the observed values of $\B{A}$ in the target domains.\footnote{This trivial observation is not novel; see e.g.\ \citep[Ch.~7, p.~164,][]{Spirtes2000}. It also follows as a special case of \citep[Theorem 2,][]{PearlBareinboim2011}. The main novelty of this work is the proposed strategy to identify such separating sets.}

We will now discuss the effect of the choice of $\B{A}$ on the quality of the predictions.
For simplicity of the exposition, we make use of the squared loss function and 
look at the asymptotic case, ignoring finite-sample issues. When predicting $Y$ from a subset of features $\B{A} \subseteq \B{V} \setminus \{Y,C_1\}$ (that may or may not be separating), the optimal predictor is defined as the function $\hat Y$ mapping from the range of possible values
of $\B{A}$ to the range of possible values of $Y$ that minimizes the \emph{target domains risk}
  $\Exp \big( (Y - \hat Y(\B{A}))^2 \given C_1=1 \big)$,
and is given by the
conditional expectation (regression function) 
$\hat Y_{\B{A}}^1(\B{a}) \defeq \Exp(Y \given \B{A}=\B{a},C_1=1)$.
Since $Y$ is not observed in the target domains, we cannot directly estimate this regression function from the data.

One approach that is often used in practice is to ignore the difference in distribution between source and target domains,
and use instead the predictor $\hat Y_{\B{A}}^0(\B{a}) \defeq \Exp(Y \given \B{A}=\B{a},C_1=0)$,
which minimizes the 
\emph{source domains risk} $\Exp \big( (Y - \hat Y)^2 \given C_1=0 \big)$. This approximation introduces a bias
$\hat Y_{\B{A}}^1 - \hat Y_{\B{A}}^0$
that we will refer to as the \emph{transfer bias} (when predicting $Y$ from $\B{A}$).
When ignoring that source domains and target domains have different distributions, any 
standard machine learning method can be used to predict $Y$ from $\B{A}$. As the 
transfer bias can become arbitrarily large (as we have seen in Example~\ref{ex:FSfailure}), the prediction accuracy of
this solution strategy may be arbitrarily bad (even in the infinite-sample limit).

Instead, we propose to only predict $Y$ from $\B{A}$ when the set $\B{A}$ of
features satisfies the following \emph{separating set} property:
\begin{equation}\label{eq:separatingset}
C_1 \dsep Y \given \B{A}\ [\C{G}],
\end{equation}
i.e., it d-separates $C_1$ from $Y$ in $\C{G}$. By the Markov assumption, this implies $C_1 \indep Y \given \B{A}\ [\Prb(\B{V})]$. 
In other words (as already mentioned above), for separating sets, 
the distribution of $Y$ conditional on $\B{A}$ is \emph{invariant} under transferring from the source domains to the target domains, i.e., $\Prb(Y \given \B{A}, C_1 = 0) = \Prb(Y \given \B{A}, C_1 = 1)$.
By virtue of this invariance, regression functions are identical for the source domains and target domains, i.e., $\hat Y_{\B{A}}^0 = \hat Y_{\B{A}}^1$, and hence also the source domains and target domains risks are identical when using the predictor $\hat Y_{\B{A}}^0$:
\begin{equation}\label{eq:guarantee}
  C_1 \dsep Y \given \B{A}\ [\C{G}] \implies \Exp \big( (Y - \hat Y_{\B{A}}^0)^2 \given C_1=1 \big) = \Exp \big( (Y - \hat Y_{\B{A}}^0)^2 \given C_1=0 \big).
\end{equation}
The r.h.s.\ can be estimated from the source domains data, and the l.h.s.\ equals the generalization
error to the target domains when using the predictor $\hat Y_{\B{A}}^0$ trained on the source domains
(which equals the predictor $\hat Y_{\B{A}}^1$ that one could obtain if all target domains data, including
the values of $Y$, were observed).\footnote{Note that this equation only holds asymptotically; for finite samples, in addition to the transfer from source domains
to target domains, we have to deal with the generalization from empirical to population distributions and from
the covariate shift if $\Prb(\B{A} \given C_1=1) \ne \Prb(\B{A} \given C_1=0)$ \cite[see e.g.][]{Mansour++_2009}.}
Although this approach leads to zero transfer bias, it introduces another bias: by using only a subset of the features $\B{A}$, 
rather than \emph{all available} features $\B{V} \setminus \{C_1,Y\}$, we may miss relevant information to predict $Y$. 
We refer to this bias as the \emph{incomplete information bias},
$\hat Y_{\B{V} \setminus \{Y,C_1\}}^1 - \hat Y_{\B{A}}^1$.

The total bias when using $\hat Y_{\B{A}}^0$ to predict $Y$ is the sum of the transfer bias and the incomplete information bias:
$$\underbrace{\hat Y_{\B{V} \setminus \{Y,C_1\}}^1 - \hat Y_{\B{A}}^0}_{\text{total bias}} = \underbrace{(\hat Y_{\B{A}}^1 - \hat Y_{\B{A}}^0)}_{\text{transfer bias}} + \underbrace{(\hat Y_{\B{V} \setminus \{Y,C_1\}}^1 - \hat Y_{\B{A}}^1)}_{\text{incomplete information bias}}.$$
For some problems, one may be better off by simply ignoring the transfer bias and minimizing the incomplete information bias, while for other problems, it is crucial to take
the transfer into account to obtain small generalization errors. In that situation, we could use any subset $\B{A}$ for prediction 
that satisfies the separating set property \eref{eq:separatingset}, implying zero transfer bias; obviously, the best predictions are then obtained by selecting a separating subset that also minimizes the source domains risk (i.e., minimizes the incomplete information bias). 
We conclude that this strategy of selecting a subset $\B{A}$ to predict $Y$ may yield an asymptotic guarantee on the prediction error by \eref{eq:guarantee}, whereas simply ignoring the shift in distribution may lead to unbounded prediction error, since the transfer bias could be arbitrarily large in the worst case scenario.

\subsection{Identifiability of Separating Feature Sets}


For the strategy of selecting the best separating sets of features as discussed in Section~\ref{sec:causal_feature_selection}, we need to find one or more sets $\B{A} \subseteq \B{V} \setminus \{C_1,Y\}$ that satisfy \eref{eq:separatingset}. Of course, the problem is that we cannot directly test this in the data, because the values of $Y$ are missing for $C_1 = 1$. Note that also Assumption~\ref{ass:ctl}(\ref{ass:ctl_faithfulness_source}) cannot be directly used here, because it only applies when $C_1$ is \emph{not} in $\B{A} \cup \B{B}$.
When the causal graph $\C{G}$ is known, it is easy to verify whether \eref{eq:separatingset} holds directly using d-separation.
Here we address the more challenging setting in which the causal graph and the targets of the interventions are (partially) unknown.\footnote{Another option, proposed by \citet{Rojas-Carulla++_2018}, is to \emph{assume} that 
if $p(Y \given \B{A})$ is invariant across all source domains (i.e., $p(Y \given \B{A}, C_1=0, C_{\setminus 1}=c) = p(Y \given \B{A}, C_1=0)$ for all $c$), then the same holds across all source \emph{and} target domains (i.e., $p(Y \given \B{A}, C_1=1) = p(Y \given \B{A}, C_1=0, C_{\setminus 1}=c)$ for all $c$). This assumption can be violated in some simple cases, e.g. see Example~\ref{ex:identifiable_but_fs_fails}.} 
Conceptually, one could estimate a set of possible causal graphs by using a causal discovery algorithm (for example, extending any standard method to deal with the missing conditional independence tests in $C_1=1$), and then read off separating sets from these graphs. 
In practice, it is not necessary to estimate completely these causal graphs: we only need to know enough about them to verify or falsify whether a given set of features separates $C_1$ from $Y$. 
The following example (with details in the Supplementary Material) illustrates a case where such reasoning allows us to identify a separating set.

\begin{example}\label{ex:identifiable_but_fs_fails}
  Assume that Assumptions~\ref{ass:jci_backgroundknowledge} and \ref{ass:ctl} hold for two context variables $C_1,C_2$ and three system variables $X_1,X_2,X_3$ with $Y \defeq X_2$.
  If the following conditional (in)dependences all hold in the source domains:
  \begin{equation}\label{eq:identifiable_but_fs_fails}C_2 \indep X_2 \given X_1\ [C_1 = 0], \qquad C_2 \notindep X_2 \given \emptyset\ [C_1 = 0], \qquad C_2 \indep X_3 \given X_2 \ [C_1=0],
  \end{equation}
  then $C_1 \dsep X_2 \given X_1\ [\C{G}]$, i.e., $\{X_1\}$ is a separating set for $C_1$ and $X_2$. One possible causal graph leading to those (in)dependences is provided in Figure~\ref{fig:task} (the others are shown in Figure~1(a) in the Supplementary Material). 
  For that ADMG, and given enough data, feature selection applied to the source domains data will generically select $\{X_1,X_3\}$ as the optimal set of features for predicting $Y \defeq X_2$, which can lead to an arbitrarily large prediction error. On the other hand, the set $\{X_1\}$ is separating in any ADMG satisfying (\ref{eq:identifiable_but_fs_fails}), so using it to predict $Y$ leads to zero transfer bias, and therefore provides a guarantee on the target domains risk (i.e., it provides an upper bound on the optimal target domains risk, which can be estimated from the source domains data).
\end{example}

Rather than characterizing by hand all possible situations in which a separating set can be identified (like in Example~\ref{ex:identifiable_but_fs_fails}), in this work we delegate the causal inference to an automatic theorem prover. Intuitively, the idea is to provide the automatic theorem prover with the conditional (in)dependences that hold in the data, in combination with an encoding of Assumptions~\ref{ass:jci_backgroundknowledge} and \ref{ass:ctl} into logical rules, and ask the theorem prover whether it can prove that $C_1 \dsep Y \given \B{A}$ holds for a candidate set $\B{A}$ from the assumptions and provided conditional (in)dependences. There are three possibilities: either it can prove the query (and then we can proceed to predict $Y$ from $\B{A}$ and get an estimate of the target domains risk), or it can disprove the query (and then we know $\B{A}$ will generically give predictions that suffer from an arbitrarily large transfer bias), or it can do neither (in which case hopefully another subset $\B{A}$ can be found that does provably satisfy \eref{eq:separatingset}).

\subsection{Algorithm}\label{sec:implementation} 


A simple (brute-force) algorithm that finds the best separating set as described in Section~\ref{sec:causal_feature_selection} is the following. By using a standard feature selection method, produce a ranked list of subsets $\B{A} \subseteq \B{V} \setminus \{Y,C_1\}$, ordered ascendingly with respect to the empirical source domains risks. Going through this list of subsets (starting with the one with the smallest empirical source domains risk), test whether the separating set property can be inferred from the data by querying the automated theorem prover. If \eref{eq:separatingset} can be shown to hold, use that subset $\B{A}$ for prediction of $Y$ and stop; if not, continue with the next candidate subset $\B{A}$ in the list. If no subset satisfies \eref{eq:separatingset}, abstain from making a prediction.\footnote{Abstaining from predictions can be advantageous when trading off recall and precision. If a prediction \emph{has} to be made, we can fall back on some other method or simply accept the risk that the transfer bias may be large.}

An important consequence of Assumption~\ref{ass:ctl}(\ref{ass:ctl_faithfulness_source}) is that it enables us to transfer conditional independence involving the target variable from the source domains to the target domains (proof provided in the Supplementary Material): 
\begin{proposition}\label{prop:transferCI}
  Under Assumption~\ref{ass:ctl}, 
  \begin{equation*}
    \B{A} \indep \B{B} \given \B{S}\ [C_1=0] \iff \B{A} \indep \B{B} \given \B{S} \cup \{C_1\} \iff \B{A} \dsep \B{B} \given \B{S} \cup \{C_1\}\ [\C{G}] 
  \end{equation*}
  for subsets $\B{A},\B{B},\B{S} \subseteq \B{V}$ such that their union contains $Y$ but not $C_1$.
\end{proposition}

To test the separating set condition \eref{eq:separatingset}, 
we use the approach proposed by \citet{antti}, where we simply add the JCI assumptions (Assumption~\ref{ass:jci_backgroundknowledge}) as constraints on the optimization problem, in addition to the domain-adaptation specific assumption that $C_1 \to Y \notin \C{G}$ (Assumption~\ref{ass:ctl}(\ref{ass:ctl_no_direct_effect})). As inputs we use all directly testable conditional independence test p-values $p_{\B{A} \indep \B{B} \given \B{S}}$ in the pooled data (when
$Y \not \in \B{A} \cup \B{B} \cup \B{S}$) and all those resulting from Proposition~\ref{prop:transferCI} from the source domains data only (if $Y \in \B{A} \cup \B{B} \cup \B{S}$). If background knowledge on intervention targets or the causal graph is available, it can easily be added as well.
We use the method proposed by \citet{ACI} to query for the confidence of whether some statement (e.g., $Y \indep C_1 \given \B{A}$) is true or false. The results of \citet{ACI} show that this approach is sound under oracle inputs, and asymptotically consistent whenever the statistical conditional independence tests used are asymptotically consistent. In other words, in this way the probability of wrongly deciding whether a subset $\B{A}$ is a separating set converges to zero as the sample size increases.
We chose this approach because it is simple to implement on top of existing open source code.\footnote{We build on the source code provided by \citet{ACI} which in turn extends the source code provided by \citet{antti}. The full source code of our implementation and the experiments is available online at \url{https://github.com/caus-am/dom_adapt}.} Note that the computational cost quickly increases with the number of variables, limiting the number of variables that can be considered simultaneously.

One remaining issue is how to predict $Y$ when an optimal separating set $\B{A}$ has been found. As the distribution of $\B{A}$ may shift when transferring from source domains to target domains, this means that there is a \emph{covariate shift} to be taken into account when predicting $Y$. Any method (e.g., least-squares regression) could in principle be used to predict $Y$ from a given set of covariates, but it is advisable to use a prediction method that works well under covariate shift, e.g., \citep{sugiyama2008direct}.



\section{Evaluation}\label{sec:results}

We perform an evaluation on both synthetic data and a real-world dataset based on a causal inference challenge.\footnote{Part of the CRM workshop on Statistical Causal Inference and Applications to Genetics, Montreal, Canada (2016). See also \url{http://www.crm.umontreal.ca/2016/Genetics16/competition_e.php}} The latter dataset consists of hematology-related measurements from the International Mouse Phenotyping Consortium (IMPC), which collects measurements of phenotypes of mice with different single-gene knockouts.

In both evaluations we compare a standard feature selection method (which uses Random Forests) with our method that builds on top of it and selects from its output the best separating set. First, we score all possible subsets of features by their out-of-bag score using the implementation of Random Forest Regressor from \texttt{scikit-learn} \citep{scikit-learn} with default parameters. For the baseline we then select the best performing subset and predict $Y$.
Instead, for our proposed method we try to find a subset of features $\B{A}$ that is also a separating set, starting from the subsets with the best scores. To test whether $\B{A}$ is a separating set, we use the method described in Section~\ref{sec:implementation}, using the ASP solver \texttt{clingo 4.5.4} \citep{clingo}. We provide as inputs the independence test results from a partial correlation test with significance level $\alpha=0.05$ and combine it with the weighting scheme from \citet{ACI}. We then use the first subset $\B{A}$ in the ranked list of predictive sets of features found by the Random Forest method for which the confidence that $C_1 \dsep Y \given \B{A}$ holds is positive. If there is no set $\B{A}$ that satisfies this criterion, then we abstain from making a prediction.

For the synthetic data, we generate randomly 200 linear acyclic models with latent variables and Gaussian noise, each with three system variables, and sample $N$ data points each for the observational and two experimental domains, where we simulate soft interventions on randomly selected targets, with different sizes of perturbations.
We randomly select which of the two context variables will be $C_1$ and which of the three system variables will be $Y$. We disallow direct effects of $C_1$ on $Y$, and enforce that no intervention can directly affect all variables simultaneously. More details on how the data were simulated are provided in the Supplementary Material.
Figure~\ref{fig:synthresults} shows a boxplot of the $L_2$ loss of the predicted $Y$ values with respect to the true values for both the baseline and our method, considering the 121 cases out of 200 in which our method does produce an answer. In particular, Figure~\ref{fig:synthresults} considers the case of $N=1000$ samples per regime and interventions that all produce a large perturbation. In the Supplementary Material we show that results improve with more samples, both for the baseline, but even more so for our method, since the quality of the conditional independence tests improves. We also show that, according to expectations, if the target distribution is very similar to the source distributions, i.e., the transfer bias is small, our method does not provide any benefit and seems to perform worse than the baseline. Conversely, the larger the intervention effect, the bigger the advantage of using our method.

For the real-world dataset, we select a subset of the variables considered in the CRM Causal Inference Challenge. Specifically, for simplicity we focus on 16 phenotypes that are not deterministically related to each other. The dataset contains measurements for 441 ``wild type'' mice and for about 10 ``mutant'' mice for each of 13 different single gene knockouts. We then generate 1000 datasets by randomly selecting subsets of 3 variables and 2 gene knockout contexts, and always include also ``wild type'' mice. For each dataset we randomly choose $Y$ and $C_1$, and leave out the observed values of $Y$ for $C_1=1$. Figure~\ref{fig:miceresults} shows a boxplot of the $L_2$ loss of the predicted $Y$ values with respect to the real values for the baseline and our method. Given the small size of the datasets, this is a very challenging problem. In this case, our method abstains from making a prediction for 170 cases out of 1000 but performs similarly to the baseline on the remaining cases.

\begin{figure}[t]
\centering
  \begin{subfigure}[t]{0.4\textwidth}%
    \centering%
    \includegraphics[width=\textwidth]{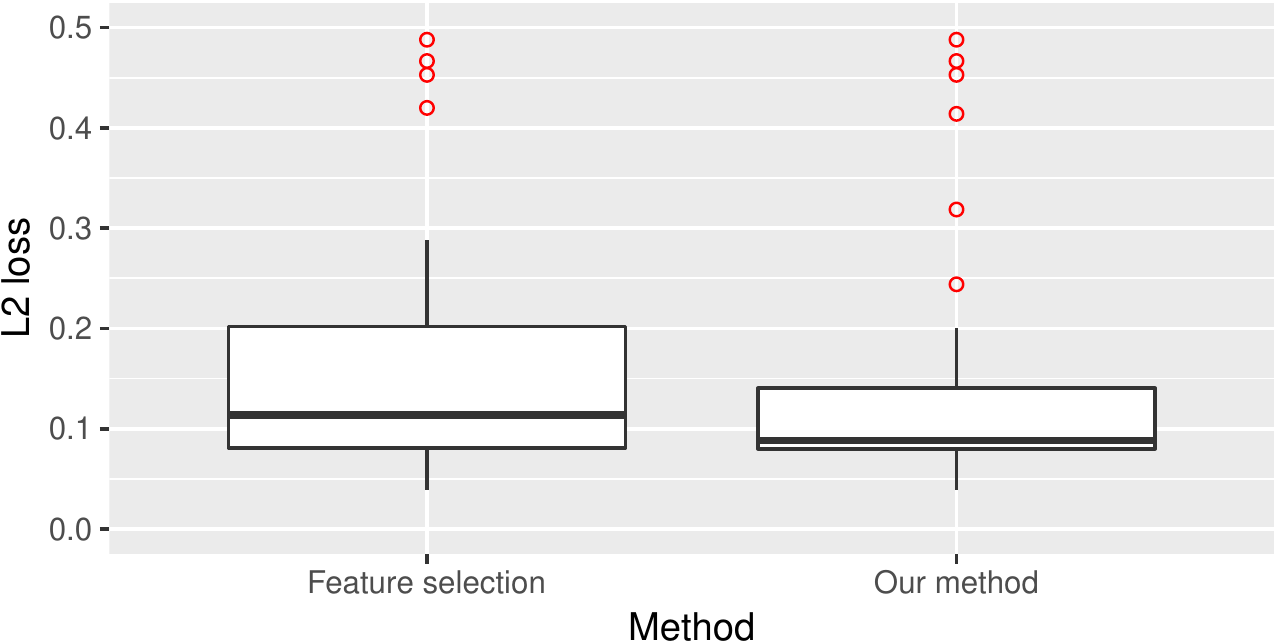}%
    \caption{Synthetic data with $N=1000$ samples and a large perturbation}\label{fig:synthresults}
  \end{subfigure}\qquad
  \begin{subfigure}[t]{0.4\textwidth}%
    \centering%
    \includegraphics[width=\textwidth]{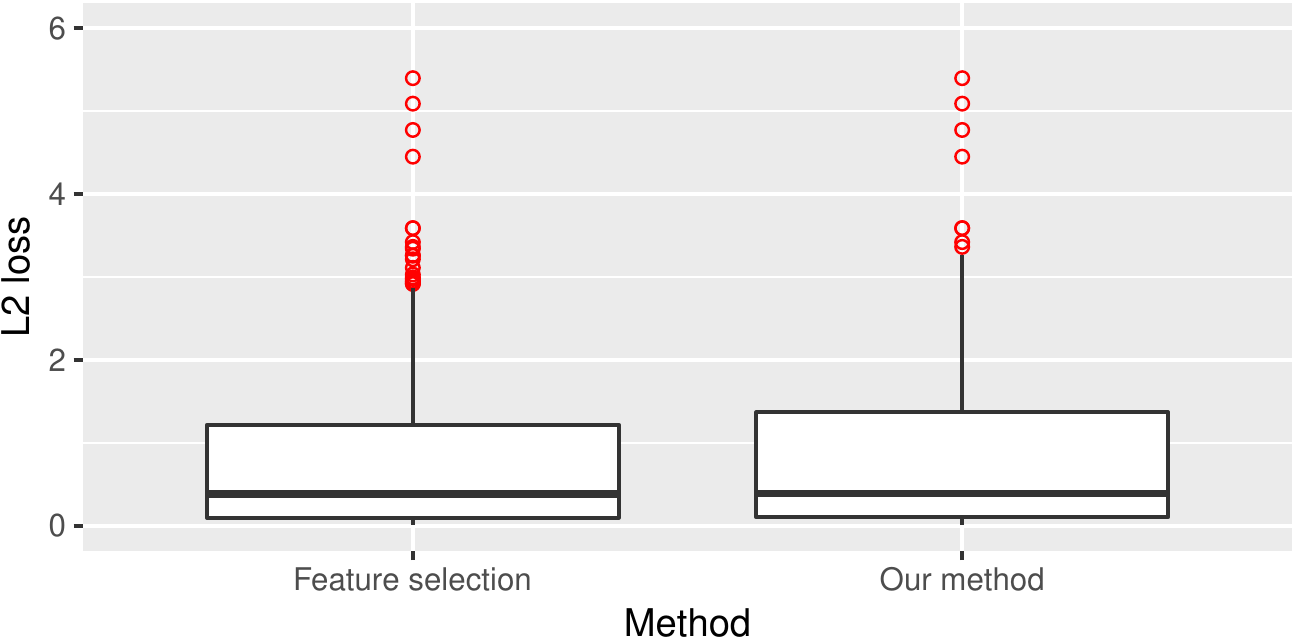}%
    \caption{Real-world data}\label{fig:miceresults}
  \end{subfigure}
  \caption{Evaluation results (see main text and Supplementary Material for details).
\label{fig:results}}
\end{figure}

\section{Discussion and Conclusion}\label{sec:discussion}

We have defined a general class of causal domain adaptation problems and proposed a method that can identify sets of features that lead to transferable predictions. Our assumptions are quite general and in particular do not require the causal graph or the intervention targets to be known. The method gives promising results on simulated data. It is straightforward to extend our method to the cyclic case by making use of the results by \citet{ForreMooij_UAI_18}.
More work remains to be done on the implementation side, for example, scaling up to more variables. Currently, our approach can handle about seven variables on a laptop computer, and with recent advances in exact causal discovery algorithms \citep[e.g.,][]{RantanenHyttinenJarvisalo2018}, a few more variables would be feasible. For scaling up to dozens of variables, we plan to adapt constraint-based causal discovery algorithms like FCI \citep{Spirtes2000} to deal with the missing-data aspect of the domain adaptation task. We hope that this work will also inspire further research on the interplay between bias, variance and causality from a statistical learning theory perspective.



\newpage
\subsubsection*{Acknowledgments}
We thank Patrick Forr\'e for proofreading a draft of this work. We thank Ren\'ee van Amerongen and Lucas van Eijk for sharing their domain knowledge about the hematology-related measurements from the International Mouse Phenotyping Consortium (IMPC).
SM, TC, SB, and PV were supported by NWO, the Netherlands Organization for Scientific Research (VIDI grant 639.072.410). 
SM was also supported by the Dutch programme COMMIT/ under the Data2Semantics project.
TC was also supported by NWO grant 612.001.202 (MoCoCaDi), and EU-FP7 grant agreement n.603016 (MATRICS).
TvO and JMM were supported by the European Research Council (ERC) under the European Union's Horizon 2020 research and
innovation programme (grant agreement 639466).



{
\small

\begin{thebibliography}{29}
\providecommand{\natexlab}[1]{#1}
\providecommand{\url}[1]{\texttt{#1}}
\expandafter\ifx\csname urlstyle\endcsname\relax
  \providecommand{\doi}[1]{doi: #1}\else
  \providecommand{\doi}{doi: \begingroup \urlstyle{rm}\Url}\fi

\bibitem[Bareinboim and Pearl(2016)]{BareinboimPearl_2016}
E.~Bareinboim and J.~Pearl.
\newblock Causal inference and the data-fusion problem.
\newblock \emph{Proceedings of the National Academy of Sciences}, 113\penalty0
  (27):\penalty0 7345--7352, 2016.

\bibitem[Bongers et~al.(2018)Bongers, Peters, Sch{\"o}lkopf, and
  Mooij]{Bongers++_1611.06221v2}
S.~Bongers, J.~Peters, B.~Sch{\"o}lkopf, and J.~M. Mooij.
\newblock Theoretical aspects of cyclic structural causal models.
\newblock \emph{arXiv.org preprint}, arXiv:1611.06221v2 [stat.ME], Aug. 2018.
\newblock URL \url{https://arxiv.org/abs/1611.06221v2}.

\bibitem[Cooper(1997)]{Cooper1997}
G.~F. Cooper.
\newblock A simple constraint-based algorithm for efficiently mining
  observational databases for causal relationships.
\newblock \emph{Data Mining and Knowledge Discovery}, 1\penalty0 (2):\penalty0
  203--224, 1997.

\bibitem[Eaton and Murphy(2007)]{EatonMurphy07}
D.~Eaton and K.~Murphy.
\newblock Exact {B}ayesian structure learning from uncertain interventions.
\newblock In \emph{Proceedings of the Eleventh International Conference on
  Artificial Intelligence and Statistics, ({AISTATS}-07)}, volume~2 of
  \emph{Proceedings of Machine Learning Research}, pages 107--114, 2007.

\bibitem[Forr{\'e} and Mooij(2018)]{ForreMooij_UAI_18}
P.~Forr{\'e} and J.~M. Mooij.
\newblock Constraint-based causal discovery for non-linear structural causal
  models with cycles and latent confounders.
\newblock In \emph{Proceedings of the 34th Annual Conference on {U}ncertainty
  in {A}rtificial {I}ntelligence ({UAI}-18)}, 2018.

\bibitem[Gebser et~al.(2014)Gebser, Kaminski, Kaufmann, and Schaub]{clingo}
M.~Gebser, R.~Kaminski, B.~Kaufmann, and T.~Schaub.
\newblock \textit{Clingo} = {ASP} + control: Extended report.
\newblock Technical report, University of Potsdam, 2014.
\newblock URL \url{http://www.cs.uni-potsdam.de/wv/pdfformat/gekakasc14a.pdf}.

\bibitem[Gong et~al.(2016)Gong, Zhang, Liu, Tao, Glymour, and
  Sch{\"o}lkopf]{Gong++2016}
M.~Gong, K.~Zhang, T.~Liu, D.~Tao, C.~Glymour, and B.~Sch{\"o}lkopf.
\newblock Domain adaptation with conditional transferable components.
\newblock In \emph{Proceedings of the 33rd International Conference on Machine
  Learning ({ICML} 2016)}, volume~48 of \emph{{JMLR} Workshop and Conference
  Proceedings}, pages 2839--2848, 2016.

\bibitem[Hyttinen et~al.(2014)Hyttinen, Eberhardt, and J{\"{a}}rvisalo]{antti}
A.~Hyttinen, F.~Eberhardt, and M.~J{\"{a}}rvisalo.
\newblock Constraint-based causal discovery: Conflict resolution with answer
  set programming.
\newblock In \emph{Proceedings of the Thirtieth Conference on Uncertainty in
  Artificial Intelligence, ({UAI}-14)}, pages 340--349, 2014.

\bibitem[Hyttinen et~al.(2015)Hyttinen, Eberhardt, and
  J{\"a}rvisalo]{Hyttinen++_2015}
A.~Hyttinen, F.~Eberhardt, and M.~J{\"a}rvisalo.
\newblock Do-calculus when the true graph is unknown.
\newblock In \emph{Proceedings of the Thirty-First Conference on Uncertainty in
  Artificial Intelligence ({UAI} 2015)}, pages 395--404, 2015.

\bibitem[Magliacane et~al.(2016)Magliacane, Claassen, and Mooij]{ACI}
S.~Magliacane, T.~Claassen, and J.~M. Mooij.
\newblock Ancestral causal inference.
\newblock In \emph{In Proceedings of {A}dvances in {N}eural {I}nformation
  {P}rocessing {S}ystems, ({NIPS}-16)}, pages 4466--4474, 2016.

\bibitem[Mansour et~al.(2009)Mansour, Mohri, and Rostamizadeh]{Mansour++_2009}
Y.~Mansour, M.~Mohri, and A.~Rostamizadeh.
\newblock Domain adaptation: Learning bounds and algorithms.
\newblock In \emph{Proceedings of the Twenty-Second Annual Conference on
  Learning Theory ({COLT} 2009)}, 2009.

\bibitem[Markowetz et~al.(2005)Markowetz, Grossmann, and
  Spang]{Markowetz++2005}
F.~Markowetz, S.~Grossmann, and R.~Spang.
\newblock Probabilistic soft interventions in conditional {G}aussian networks.
\newblock In \emph{Proceedings of the Tenth International Workshop on
  Artificial Intelligence and Statistics, ({AISTATS}-05)}, pages 214--221,
  2005.

\bibitem[Mooij and Heskes(2013)]{MooijHeskes_UAI_13}
J.~M. Mooij and T.~Heskes.
\newblock Cyclic causal discovery from continuous equilibrium data.
\newblock In \emph{Proceedings of the 29th Annual Conference on {U}ncertainty
  in {A}rtificial {I}ntelligence ({UAI}-13)}, pages 431--439, 2013.

\bibitem[Mooij et~al.(2018)Mooij, Magliacane, and
  Claassen]{Mooij++_1611.10351v3}
J.~M. Mooij, S.~Magliacane, and T.~Claassen.
\newblock Joint causal inference from multiple contexts.
\newblock \emph{arXiv.org preprint}, https://arxiv.org/abs/1611.10351v3
  [cs.LG], Mar. 2018.
\newblock URL \url{https://arxiv.org/abs/1611.10351v3}.

\bibitem[Pan and Yang(2010)]{Pan2010survey}
S.~J. Pan and Q.~Yang.
\newblock A survey on transfer learning.
\newblock \emph{IEEE Transactions on Knowledge and Data Engineering},
  22\penalty0 (10):\penalty0 1345--1359, Oct. 2010.

\bibitem[Pearl(2009)]{Pearl2009}
J.~Pearl.
\newblock \emph{Causality: models, reasoning and inference}.
\newblock Cambridge University Press, 2009.

\bibitem[Pearl and Bareinboim(2011)]{PearlBareinboim2011}
J.~Pearl and E.~Bareinboim.
\newblock Transportability of causal and statistical relations: A formal
  approach.
\newblock In \emph{Proceedings of the Twenty-Fifth AAAI Conference on
  Artificial Intelligence}, pages 247--254, 2011.

\bibitem[Pedregosa et~al.(2011)Pedregosa, Varoquaux, Gramfort, Michel, Thirion,
  Grisel, Blondel, Prettenhofer, Weiss, Dubourg, Vanderplas, Passos,
  Cournapeau, Brucher, Perrot, and Duchesnay]{scikit-learn}
F.~Pedregosa, G.~Varoquaux, A.~Gramfort, V.~Michel, B.~Thirion, O.~Grisel,
  M.~Blondel, P.~Prettenhofer, R.~Weiss, V.~Dubourg, J.~Vanderplas, A.~Passos,
  D.~Cournapeau, M.~Brucher, M.~Perrot, and E.~Duchesnay.
\newblock Scikit-learn: Machine learning in {P}ython.
\newblock \emph{Journal of Machine Learning Research}, 12:\penalty0 2825--2830,
  2011.

\bibitem[Qui{\~n}onero-Candela et~al.(2009)Qui{\~n}onero-Candela, Suyiyama,
  Schwaighofer, and Lawrence]{datasetshiftML2009}
J.~Qui{\~n}onero-Candela, M.~Suyiyama, A.~Schwaighofer, and N.~D. Lawrence,
  editors.
\newblock \emph{Dataset Shift in Machine Learning}.
\newblock MIT Press, 2009.

\bibitem[Rantanen et~al.(2018)Rantanen, Hyttinen, and
  J{\"a}rvisalo]{RantanenHyttinenJarvisalo2018}
K.~Rantanen, A.~Hyttinen, and M.~J{\"a}rvisalo.
\newblock Learning optimal causal graphs with exact search.
\newblock In \emph{Proceedings of the 9th International Conference on
  Probabilistic Graphical Models ({PGM} 2018)}, volume~72 of \emph{Proceedings
  of Machine Learning Research}, pages 344--355, 2018.

\bibitem[Richardson(2003)]{Richardson2003}
T.~Richardson.
\newblock Markov properties for acyclic directed mixed graphs.
\newblock \emph{Scandinavian Journal of Statistics}, 30:\penalty0 145--157,
  2003.

\bibitem[Rojas-Carulla et~al.(2018)Rojas-Carulla, Sch{\"o}lkopf, Turner, and
  Peters]{Rojas-Carulla++_2018}
M.~Rojas-Carulla, B.~Sch{\"o}lkopf, R.~Turner, and J.~Peters.
\newblock Invariant models for causal transfer learning.
\newblock \emph{Journal of Machine Learning Research}, 19\penalty0
  (36):\penalty0 1--34, 2018.

\bibitem[Sch{\"o}lkopf et~al.(2012)Sch{\"o}lkopf, Janzing, Peters, Sgouritsa,
  Zhang, and Mooij]{Schoelkopf_et_al_ICML_12}
B.~Sch{\"o}lkopf, D.~Janzing, J.~Peters, E.~Sgouritsa, K.~Zhang, and J.~M.
  Mooij.
\newblock On causal and anticausal learning.
\newblock In \emph{Proceedings of the 29th {I}nternational {C}onference on
  {M}achine {L}earning ({ICML} 2012)}, pages 1255--1262, 2012.

\bibitem[Spirtes et~al.(2000)Spirtes, Glymour, and Scheines]{Spirtes2000}
P.~Spirtes, C.~Glymour, and R.~Scheines.
\newblock \emph{Causation, Prediction, and Search}.
\newblock MIT press, 2nd edition, 2000.

\bibitem[Storkey(2009)]{Storkey2009}
A.~Storkey.
\newblock When training and test sets are different: Characterizing learning
  transfer.
\newblock In \emph{Dataset Shift in Machine Learning}, chapter~1, pages 3--28.
  MIT Press, 2009.

\bibitem[Sugiyama et~al.(2008)Sugiyama, Nakajima, Kashima, Buenau, and
  Kawanabe]{sugiyama2008direct}
M.~Sugiyama, S.~Nakajima, H.~Kashima, P.~V. Buenau, and M.~Kawanabe.
\newblock Direct importance estimation with model selection and its application
  to covariate shift adaptation.
\newblock In \emph{In Proceedings of {A}dvances in {N}eural {I}nformation
  {P}rocessing {S}ystems ({NIPS}-08)}, pages 1433--1440, 2008.

\bibitem[Tian and Pearl(2001)]{TianPearl2001}
J.~Tian and J.~Pearl.
\newblock Causal discovery from changes.
\newblock In \emph{Proceedings of the 17th Conference in Uncertainty in
  Artificial Intelligence, ({UAI}-01)}, 2001.

\bibitem[Zhang et~al.(2013)Zhang, Sch{\"o}lkopf, Muandet, and
  Wang]{ZhangSMW2013}
K.~Zhang, B.~Sch{\"o}lkopf, K.~Muandet, and Z.~Wang.
\newblock Domain adaptation under target and conditional shift.
\newblock In \emph{Proceedings of the 30th International Conference on Machine
  Learning}, volume~28 of \emph{Proceedings of Machine Learning Research},
  pages 819--827, 2013.

\bibitem[Zhang et~al.(2015)Zhang, Gong, and Sch{\"o}lkopf]{ZhangGS2015}
K.~Zhang, M.~Gong, and B.~Sch{\"o}lkopf.
\newblock Multi-source domain adaptation: {A} causal view.
\newblock In \emph{Proceedings of the Twenty-Ninth {AAAI} Conference on
  Artificial Intelligence}, pages 3150--3157, 2015.

\end{thebibliography}

}

\newpage
\appendix
\section{Supplementary material}

\subsection{Stronger assumption}

We prove that Assumption \ref{ass:ctl}(\ref{ass:ctl_faithfulness_source}) is a weakened version of two more standard assumptions, i.e., the causal Markov and faithfulness assumptions in both source and target domains separately. 
Note that assuming these two assumptions instead of Assumption~2(ii) implies we cannot have perfect interventions in the target domain, which is otherwise allowed.

\begin{proposition}\label{prop:ctl_faithfulness_source_and_target}
Assumption \ref{ass:ctl}(\ref{ass:ctl_faithfulness_source}) 
is implied by the following assumption:
\begin{compactenum}[(a)]
\item the pooled source domains distribution $\Prb(\B{V} \given C_1 = 0)$ is Markov and faithful to $\C{G}^{\setminus C_1}$, and
\item the pooled target domains distribution $\Prb(\B{V} \given C_1 = 1)$ is Markov and faithful to $\C{G}^{\setminus C_1}$,
\end{compactenum}
where $\C{G}^{\setminus C_1}$ denotes the induced subgraph of the causal graph $\C{G}$ on the
nodes $\C{V} \setminus \{C_1\}$ (i.e., it is obtained by removing $C_1$ and all edges involving $C_1$ from 
the causal graph $\C{G}$).
\end{proposition}
\begin{proof}
Let $\B{A},\B{B},\B{S} \subseteq \B{V} \setminus \{C_1\}$. By assumption, we have that
  $$\B{A} \indep \B{B} \given \B{S}\ [C_1 = c] \iff \B{A} \dsep \B{B} \given \B{S}\ [\C{G}^{\setminus C_1}]$$
  holds for both $c=0,1$, which directly gives Assumption~\ref{ass:ctl}(\ref{ass:ctl_faithfulness_source}).
\end{proof}

\subsection{Other proofs}

\begin{proof}[Proof of Proposition~\ref{prop:transferCI}]
  First of all, 
  $\B{A} \dep \B{B} \given \B{S}\ [C_1=0]$ implies (by definition) $\B{A} \dep \B{B} \given \B{S} \cup \{C_1\}$. Second, $\B{A} \indep \B{B} \given \B{S}\ [C_1=0]$ implies (by assumption) $\B{A} \indep \B{B} \given \B{S}\ [C_1=1]$, and taken together,
  we get $\B{A} \indep \B{B} \given \B{S} \cup \{C_1\}$. By the Markov and faithfulness assumption (Assumption~\ref{ass:ctl}(i)), this holds iff $\B{A} \dsep \B{B} \given \B{S} \cup \{C_1\}\ [\C{G}]$.
\end{proof}

\begin{proof}[Proof of Example~\ref{ex:identifiable_but_fs_fails}]
\begin{figure}[b]
  \centering
  \begin{subfigure}[b]{0.32\textwidth}%
    \centering\captionsetup{width=.7\linewidth}%
    \begin{tikzpicture}
        \node[var] (C2) at (1,2) {$C_2$};
        \node[var] (A) at (0,1) {$X_1$};
        \node[var] (Y) at (0,0) {$X_2$};
        \draw[arr] (C2) edge (A);
        \draw[arr] (A) edge (Y);
    \end{tikzpicture}
    \caption{Marginal ADMG $\C{G}'$.\label{fig:example2_G'}}
  \end{subfigure}%
  \begin{subfigure}[b]{0.30\textwidth}%
    \centering\captionsetup{width=.7\linewidth}%
    \begin{tikzpicture}
      \node[var] (C1) at (-1,2) {$C_1$};
      \node[var] (C2) at (1,2) {$C_2$};
      \node[var] (A) at (0,1) {$X_1$};
      \node[var] (Y) at (0,0) {$X_2$};
      \draw[arr,dashed] (C1) edge (A);
      \draw[arr] (C2) edge (A);
      \draw[arr] (A) edge (Y);
      \draw[biarr,bend left] (C1) edge (C2);
    \end{tikzpicture}
    \caption{Set of candidate marginal ADMGs $\C{G}''$.\label{fig:example2_G''}}
  \end{subfigure}%
  \begin{subfigure}[b]{0.30\textwidth}%
    \centering\captionsetup{width=.7\linewidth}%
    \begin{tikzpicture}
      \node[var] (C1) at (-1,2) {$C_1$};
      \node[var] (C2) at (1,2) {$C_2$};
      \node[var] (A) at (0,1) {$X_1$};
      \node[var] (Y) at (0,0) {$X_2$};
      \node[var] (Z) at (-1,-1) {$X_3$};
      \draw[arr,dashed] (C1) edge (A);
      \draw[arr] (C2) edge (A);
      \draw[arr] (A) edge (Y);
      \draw[biarr,bend left] (C1) edge (C2);
      \draw[arr, dashed] (C1) edge (Z);
      \draw[arr, dashed] (Y) edge (Z);
    \end{tikzpicture}
  \caption{Set of candidate \mbox{ADMGs} $\C{G}$.\label{fig:example2_G}}
  \end{subfigure}%
  \caption{ADMGs for proof of Example~\ref{ex:identifiable_but_fs_fails}. Each dashed edge can either be present or absent.}
\end{figure}
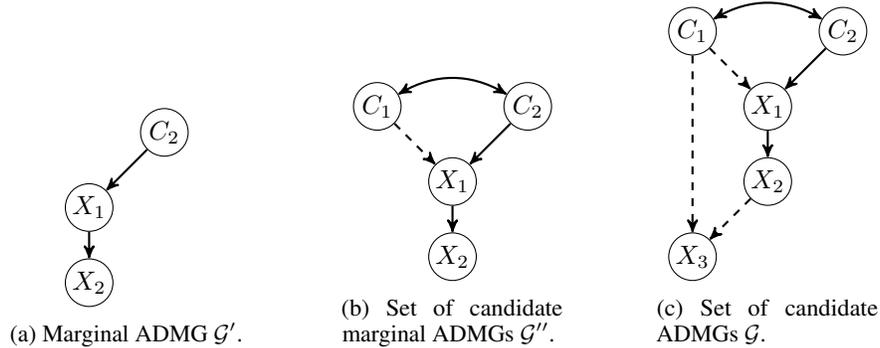
In the JCI setting, we assume that in the full ADMG $\C{G}$ over variables $\{C_1,C_2,X_1,X_2,X_3\}$, $C_1$ and $C_2$ are confounded and not caused by system variables $X_1,X_2,X_3$. Furthermore, no pair of system variable and context variables is confounded.

In the context $[C_1 =0]$, if the conditional independences $C_2 \indep X_2 \given X_1\ [C_1 = 0]$ and $C_2 \notindep X_2 \given \emptyset\ [C_1 = 0]$ hold, then we can also derive that $C_2 \notindep X_1 \given \emptyset\ \ [C_1 = 0]$, for example using Rule (9) from \cite{ACI}. 
Moreover, we know that $C_2$ is not caused by $X_1$ and $X_2$, or in other words $X_1 \notcauses C_2$ and $X_2 \notcauses C_2$. 
Thus we conclude that $(C_2,X_1,X_2)$ is an LCD triple \citep{Cooper1997} in the context $C_1 = 0$. Since in addition, in this case $C_2$ and $X_1$ are unconfounded, 
  the marginal ADMG $\C{G'}$ on $\{C_2,X_1,X_2\}$ (in the context $C_1=0$, and hence by Proposition~\ref{prop:transferCI} in all contexts) must be given by Figure~\ref{fig:example2_G'}.

Therefore, the extended marginal ADMG $\C{G''}$ on variables $\{C_1,C_2,X_1,X_2\}$ must also have a directed path from $C_2$ to $X_1$ and from $X_1$ to $X_2$. 
$C_1$ cannot be on these paths, as none of the variables causes $C_1$, and therefore $\C{G''}$ also contains the directed edges $C_2 \to X_1$ and $X_1 \to X_2$. Moreover, $\C{G''}$ cannot contain any edge between $C_2$ and $X_2$, nor a bidirected edge between $X_1$ and $X_2$, because that would violate the conditional independence. By construction, in the JCI setting there is a bidirected edge between $C_1$ and $C_2$, and that is the only bidirected edge connecting to $C_1$ or $C_2$. As we assumed there is no direct effect of $C_1$ on target $X_2$, there is no edge between $C_1$ and $X_2$ in $\C{G''}$.
There is also no directed edge $X_1 \to C_1$ in $\C{G''}$, as the JCI assumption implies none of the other variables causes $C_1$. Therefore, the marginal ADMG $\C{G''}$ is given by Figure~\ref{fig:example2_G''},
either with the directed edge $C_1\to X_1$ present, or without that edge. 

If it additionally holds that $C_2 \indep X_3 \given X_2 \ [C_1 = 0]$, we have two possibilities: 
\begin{compactenum}
\item if $C_2 \indep X_3 \given \emptyset\ [C_1 = 0]$ holds, then $X_3$ is not caused by $C_2$. This means it cannot be on any directed path from $C_2$ to $X_1$, from $X_1$ to $X_2$, or be a descendant of $X_2$. Therefore the full ADMG $\C{G}$ also necessarily contains the directed edges $C_2 \to X_1$ and $X_1 \to X_2$. 
\item if $C_2 \notindep X_3 \given \emptyset\ [C_1 = 0]$ holds, then in conjunction with $C_2 \indep X_3 \given X_2 \ [C_1 = 0]$ we can derive $X_2 \causes X_3$, for example using Rule (5) from \citep{ACI}. 
This means $X_3$ must be a descendant of $X_2$ in the full ADMG $\C{G}$, which implies it cannot be on the directed path from $C_2$ to $X_1$, or on the one from $X_1$ to $X_2$.
Therefore the full ADMG $\C{G}$ also necessarily contains the directed edges $C_2 \to X_1$ and $X_1 \to X_2$. 
\end{compactenum}
Because of the independence statements and JCI assumptions, there cannot be a bidirected edge between $X_3$ and $X_1$, $X_2$, $C_1$ or $C_2$. Similarly, there cannot be directed edges from $X_3$ to one of those nodes. The edges $X_1 \to X_3$ and $C_2 \to X_3$ must also be absent.

In both cases, there can be a directed edge from $C_1$ to $X_3$. 
Therefore, the full ADMG $\C{G}$ is of the form given in Figure~\ref{fig:example2_G}.
In all cases we see that $C_1 \dsep X_2 \given X_1\ [\C{G}]$, and we conclude that $\{X_1\}$ is a valid separating set. 


If the ADMG is as in Figure~\ref{fig:task}, then a standard feature selection method would asymptotically prefer the subset $\{X_1,X_3\}$ to predict $X_2$ over the subset $\{X_1\}$ (note that the Markov blanket of $X_2$ in context $[C_1=0]$ is $\{X_1,X_3\}$).
As a result, any prediction method trained on all available features using source domain data (i.e., in context $[C_1=0]$) may incur a possibly unbounded prediction error when used to predict $X_2$ in the target domain $[C_1=1]$ (for example, if $X_3$ is an almost deterministic copy of $X_2$ if $C_1=0$, but has a drastically different distribution if $C_1=1$).
\end{proof}

\subsection{Additional results on synthetic data}

We provide more information and experimental results for the synthetic data. We adapted the simulator of \citet{antti} to our setting. We generate randomly 200 acyclic models with three system variables, two context variables, and at most two latent variables (chosen randomly, so that the number of latent variables equals 1 or 2 each with probability $1/4$, and 0 otherwise). Each latent variable has two system variables as children, while the other variables have a random number of system variables as children, where system variables must be consistent with a chosen topological ordering, and where we enforce that a context variable may not simultaneously affect all system variables. The system and latent variables are each described by a linear structural equation with independent noise terms distributed as $\mathcal{N}(0, 0.0064)$. In these equations, each variable is multiplied by a coefficient sampled from $\mathcal{N}(0.2, 0.64)$ or $\mathcal{N}(-0.2, 0.64)$ (each with probability $1/2$ per variable). The context variables each correspond to an experimental domain; in their domain, that variable equals 1, otherwise it equals 0. This way, we simulate soft interventions. In order to scale the effect of these interventions, we multiply the coefficients of the context variables by the parameter $\gamma$, varying it from 0.1 to 100. We sample $N$ data points each for the observational and two experimental domains. Moreover, we randomly select $C_1$ and $Y$ from context and system variables respectively. We disallow direct effects of $C_1$ on $Y$.

As expected, our method performs well when the target distribution is significantly different from the source distributions. Figure~\ref{fig:extra_results_ifactor} shows different settings with different scales of intervention effects. (In most graphs, the vertical axis has been adjusted to clearly show the boxplot, but leaving out the larger outliers.) In Figure~\ref{fig:synthresults_vsmall} the intervention effects are all scaled by 0.1, resulting in very similar distributions in all domains. In this case, using our method does not offer any advantage with respect to the baseline and it actually performs worse. In the other cases, using our method starts to pay off in terms of prediction accuracy, and the difference increases with the scale of the interventions, as seen in Figure~\ref{fig:synthresults_large}.

In Figure~\ref{fig:extra_results_samples}, we vary the number of samples $N$ for each regime. The results improve with more samples, especially for our method, since the quality of the conditional independence test improves, but also for the baseline. In particular, as shown in Figure~\ref{fig:synthresults_100}, the accuracy is low for $N=100$ samples, but it improves substantially with $N=1000$ samples (Figure~\ref{fig:synthresults_small}).

\begin{figure}[t]
\centering
  \begin{subfigure}[t]{0.45\textwidth}
    \centering%
    \includegraphics[width=\textwidth]{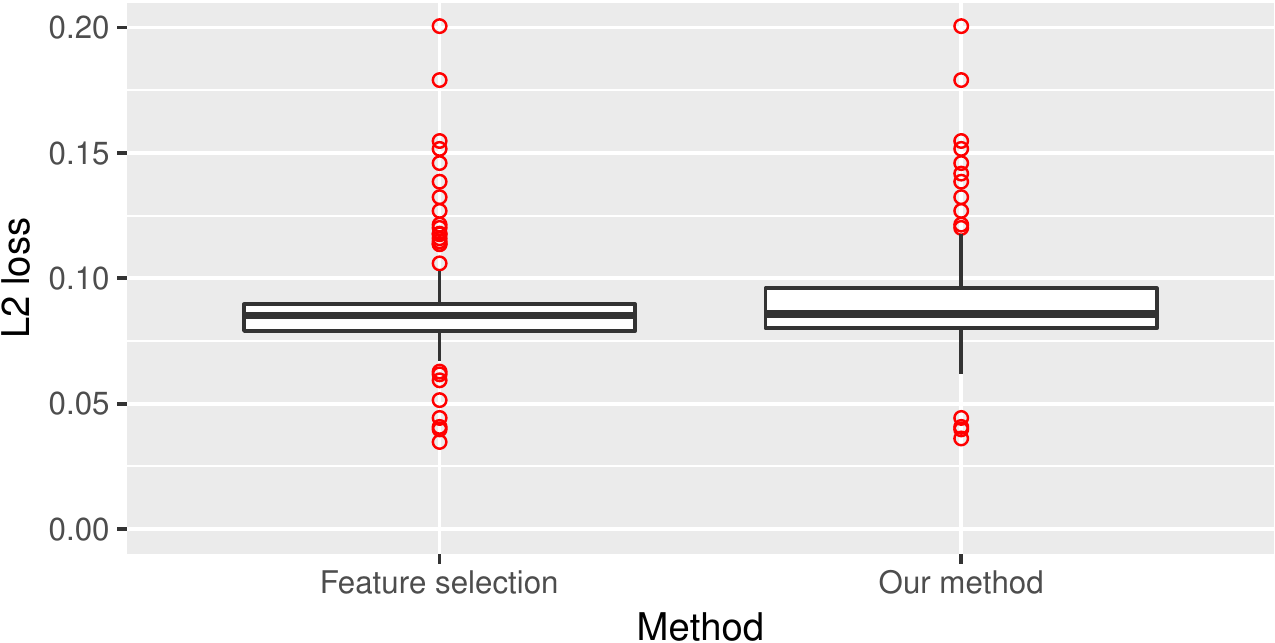}%
    \caption{Synthetic data with a small perturbation ($\gamma=0.1$) and $N=1000$ samples.\label{fig:synthresults_vsmall}}
  \end{subfigure}\qquad
  \begin{subfigure}[t]{0.45\textwidth}
    \centering%
    \includegraphics[width=\textwidth]{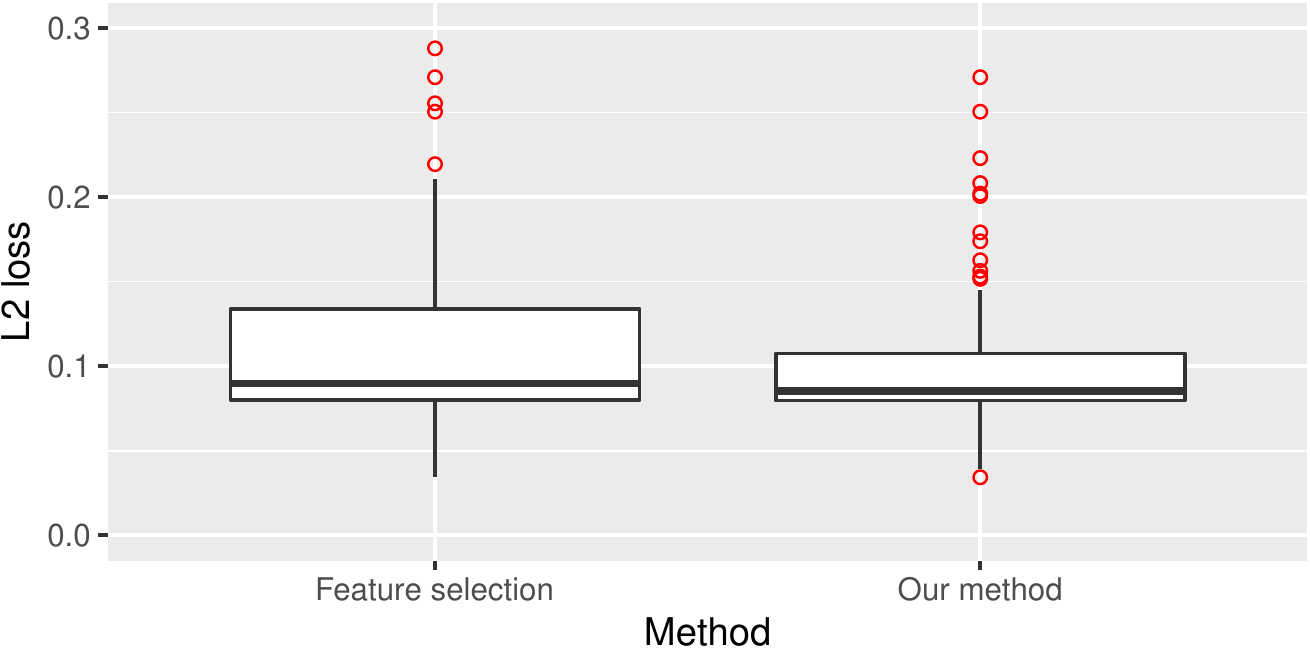}%
    \caption{Synthetic data with a medium perturbation ($\gamma=1$) and $N=1000$ samples.\label{fig:synthresults_small}}
  \end{subfigure}\\
      \begin{subfigure}[t]{0.45\textwidth}
    \centering%
    \includegraphics[width=\textwidth]{sim_200_005_1000_1000_ifactor10_zoom05}%
    \caption{Synthetic data with a large perturbation ($\gamma=10$) and $N=1000$ samples.\label{fig:synthresults_mediuml}}
  \end{subfigure}\qquad
    \begin{subfigure}[t]{0.45\textwidth}
    \centering%
    \includegraphics[width=\textwidth]{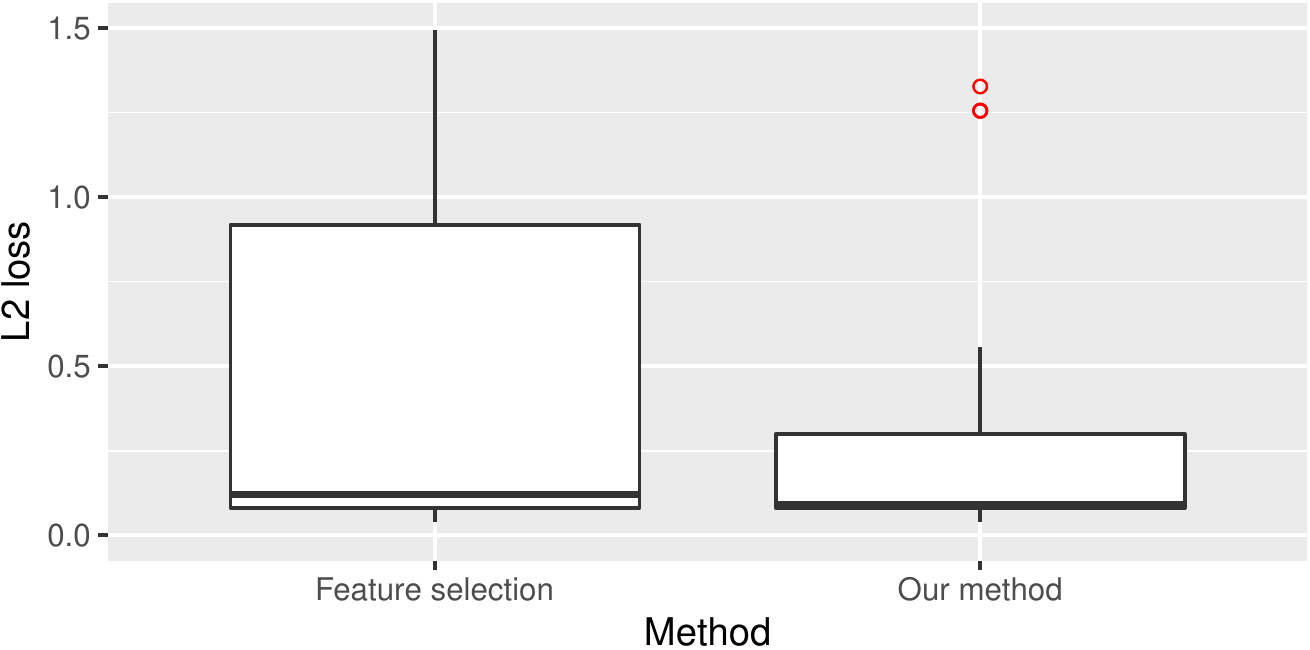}%
    \caption{Synthetic data with a very large perturbation ($\gamma=100$) and $N=1000$ samples.\label{fig:synthresults_large}}
  \end{subfigure}\qquad
  \caption{Additional results when varying the causal effect of all interventions ($\gamma$).\label{fig:extra_results_ifactor}}
\end{figure}

\begin{figure}[t]
\centering
  \begin{subfigure}[t]{0.45\textwidth}%
    \centering%
    \includegraphics[width=\textwidth]{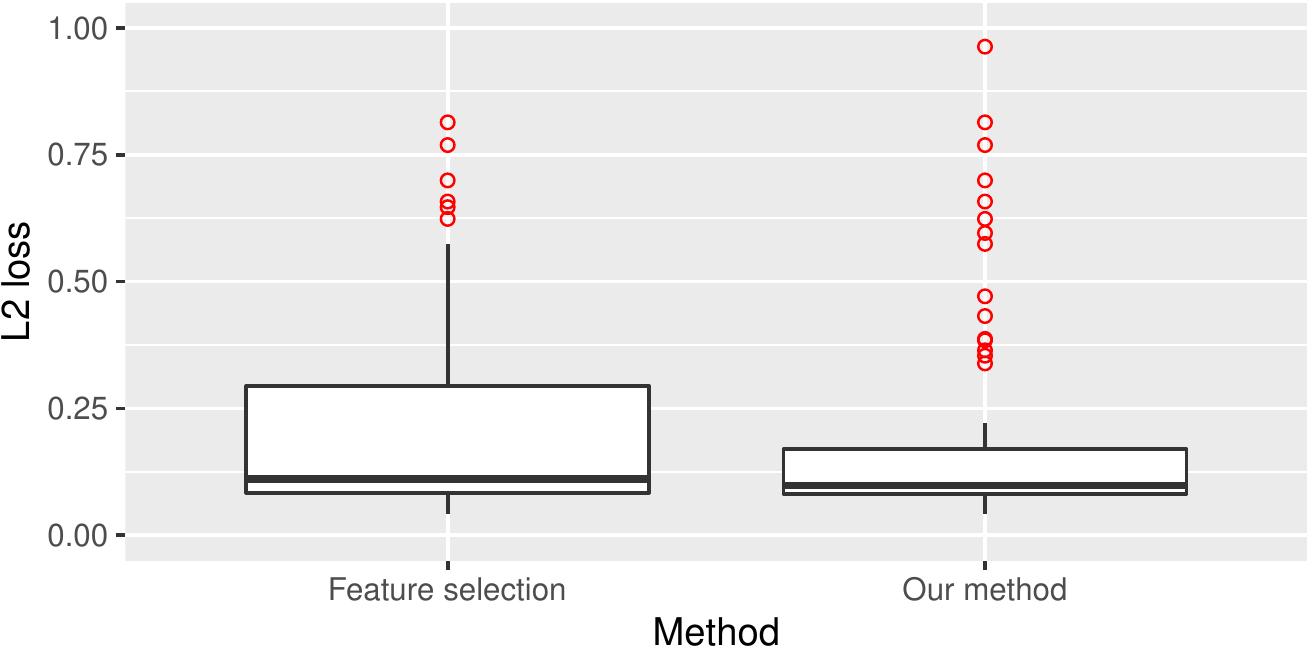}%
    \caption{Synthetic data with $N=100$ samples per regime and a large perturbation ($\gamma=10$).\label{fig:synthresults_100} }
  \end{subfigure}\qquad
  \begin{subfigure}[t]{0.45\textwidth}%
    \centering%
    \includegraphics[width=\textwidth]{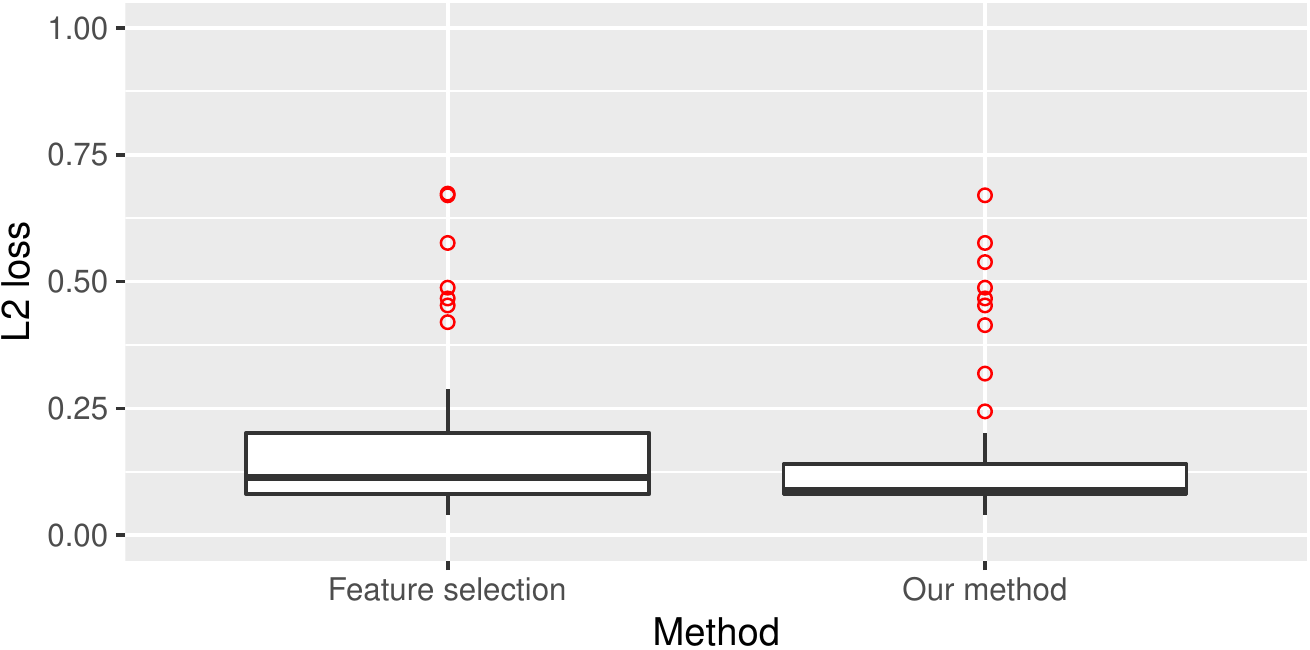}%
    \caption{Synthetic data with $N=1000$ samples per regime and a large perturbation ($\gamma=10$).\label{fig:synthresults_1000} }
  \end{subfigure}\\
  \centering
      \begin{subfigure}[t]{0.45\textwidth}%
    \centering%
    \includegraphics[width=\textwidth]{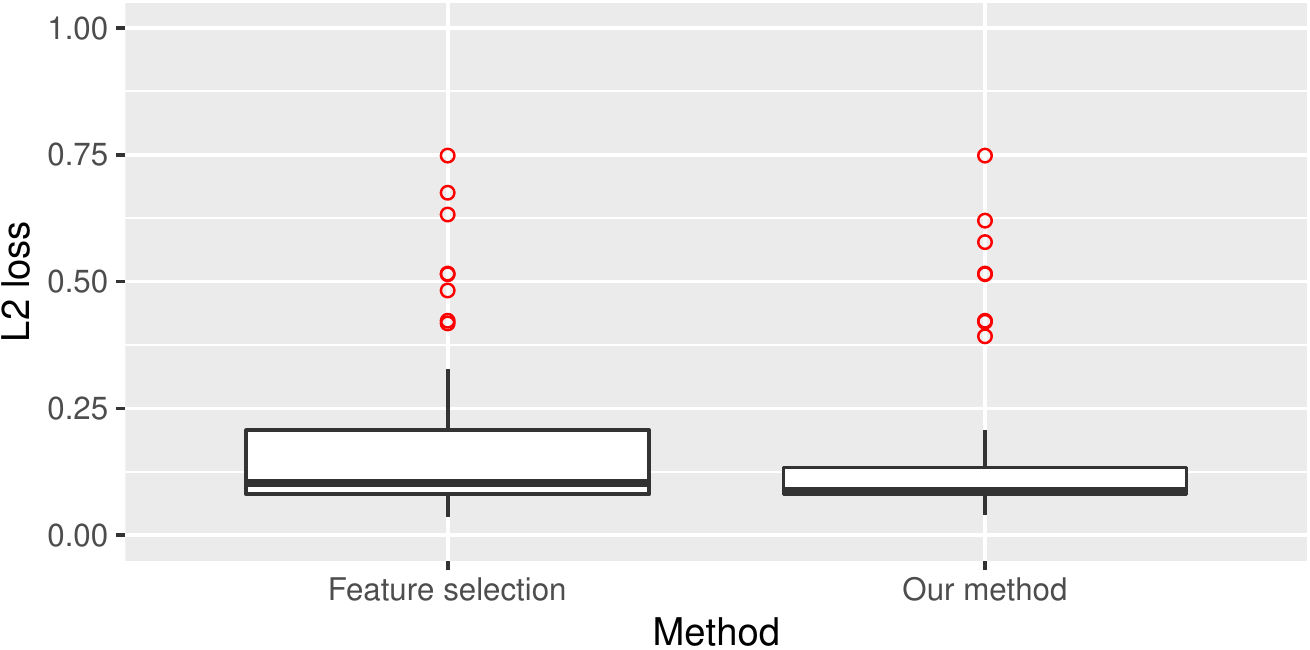}%
    \caption{Synthetic data with $N=5000$ samples per regime and a large perturbation ($\gamma=10$).\label{fig:synthresults_5000} }
    \end{subfigure}\qquad
  \caption{Additional results when varying the sample size per regime ($N$).\label{fig:extra_results_samples}}
\end{figure}


\end{document}